%% file: main.tex
\documentclass{article}                        
\usepackage[preprint]{neurips_2024}            

\usepackage[utf8]{inputenc}
\usepackage[T1]{fontenc}
\usepackage{hyperref}
\usepackage{url}
\usepackage{booktabs}
\usepackage{amsfonts}
\usepackage{amsmath}
\usepackage{amssymb}
\usepackage{amsthm}
\usepackage{graphicx}
\usepackage{nicefrac}
\usepackage{microtype}
\usepackage{xcolor}

\graphicspath{{figures/}}

\input{macros}

\title{When Can You Correct Distribution Drift in Temporal Graph
  Generation? A Sharpening--Drift Tension and an Impossibility for
  Observation-Based Correction}

\author{%
  Tianpeng Li\\
  College of New Media and Communication, Tianjin University\\
  \texttt{ltpnimeia@tju.edu.cn}
  \And
  Xuan Guo\\
  School of Artificial Intelligence, Tianjin University\\
  \texttt{guoxuan@tju.edu.cn}
  \And
  Wenjun Wang\\
  School of Artificial Intelligence, Tianjin University\\
  \texttt{wjwang@tju.edu.cn}
  \AND
  Wang Zhang\thanks{Corresponding author.}\\
  School of Artificial Intelligence, Tianjin University\\
  \texttt{wangzhang@tju.edu.cn}
  \And
  Pengfei Jiao\\
  School of Cyberspace Security,Hangzhou Dianzi University\\
  \texttt{pjiao@hdu.edu.cn}
}

\begin{document}

\maketitle

\begin{abstract}
Generative models of temporal graphs are trained on one stretch of an evolving network
and deployed on the next, and they degrade badly in the gap. We show this degradation
is derivable, general, and not fixable from observations. The masked flow-matching loss
decomposes exactly, with no independence assumption, into an irreducible entropy plus a
divergence whose derivative along the training path is positive precisely for
structures rare during training and common at deployment, diverging as their training
probability goes to zero. Empirically the trade-off is a power law with exponent
$-0.605$ ($R^2=0.9977$), and drift raises the sampler's error floor without changing
how many steps reach it: across seven well-powered conditions the drift-period marginal
error varies by at most $6\%$ over a $50\times$ range of sampling budgets, while the
floor sits $2.2\times$ to $34.3\times$ above the in-period floor. Because the
deployment period is observed, correction looks like a matter of measurement. It is
not. We prove that any corrector measurable with respect to past observations leaves at
least the conditional variance of the statistic it tracks, and that trend extrapolation
beats trusting the last observation only when $\mu^2>v(1-2\rho)$. Both premises are
measurable and both go the wrong way: the drift is trendless and mean-reverting, with a
one-step innovation as large as the drift itself. An oracle removes $60\%$ of the
error, the best observation-based corrector recovers $5.7\%$ of that, and extrapolation
is strictly worse than doing nothing clever.
\end{abstract}

\input{sections/01-intro}
\input{sections/02-related}
\input{sections/03-setup}
\input{sections/04-theorem-a}
\input{sections/05-universality}
\input{sections/06-theorem-b}
\input{sections/07-sufficiency}
\input{sections/08-discussion}
\input{sections/09-conclusion}

\bibliographystyle{plainnat}
\bibliography{refs}

\newpage
\input{sections/10-appendix}

\end{document}

%% file: macros.tex
\newcommand{\Ldfm}{\mathcal{L}_{\mathrm{DFM}}}
\newcommand{\KL}{\mathrm{KL}}
\newcommand{\TV}{\mathrm{TV}}
\newcommand{\Var}{\mathrm{Var}}
\newcommand{\Fcal}{\mathcal{F}}
\newcommand{\Ecal}{\mathbb{E}}
\newcommand{\Prob}{\mathbb{P}}
\newcommand{\qhat}{\hat{q}}
\newcommand{\phat}{\hat{p}}
\newcommand{\Ccal}{\mathcal{C}}
\newcommand{\Dcal}{\mathcal{D}}
\newcommand{\Ical}{\mathcal{I}}
\newcommand{\Hent}{H}
\newcommand{\Nmax}{N_{\max}}

\newcommand{\Ddrift}{\Dcal_{\mathrm{drift}}}   
\newcommand{\IQ}{\Ical_{Q}}                    
\newcommand{\Rsrc}{P}                          
\newcommand{\Rtgt}{Q}                          

\theoremstyle{plain}
\newtheorem{theorem}{Theorem}

\newtheorem{corollary}[theorem]{Corollary}
\newtheorem{lemma}[theorem]{Lemma}
\theoremstyle{definition}

\theoremstyle{remark}

%% file: sections/01-intro.tex
\section{Introduction}
\label{sec:intro}

Discrete diffusion and flow matching are now the default way to generate graphs
\citep{vignac2023digress,qin2024defog}, and the natural next step is to generate
\emph{temporal} graphs: train on a stretch of an evolving network, then generate the
structure that comes next. The difficulty is that ``next'' is not ``more of the same''.
Communication networks, course-enrolment logs and link-prediction benchmarks all shift
between the period a model trains on and the period it is deployed in, and a model
trained by maximum likelihood on the first has no mechanism for noticing.

The degradation is easy to underestimate. On our data a model whose in-period loss
improves by a factor of $7.8$ sees its loss on the \emph{following} period rise by a
factor of $3.4$ over the same interval. The two are not merely uncorrelated: past an
early inflection they are anti-correlated, with a rank correlation of $-1.000$ in every
seed and a power-law exponent of $-0.605$ ($R^2=0.9977$). Fitting the data you have
makes you worse at the data you will get.

\paragraph{The question, and the answer.}
Shift-induced degradation is not itself news. What makes this setting worth separate
treatment is that the shift appears \emph{observable}: at deployment we watch the
target period arrive, snapshot by snapshot, unlabelled but in full, so why not measure
how the distribution moved and re-aim the model? Our answer is that this cannot work.
An oracle seeing the current period's statistics fixes most of the damage --- we
measure a $60\%$ reduction --- but no corrector restricted to past observations
approaches it, and the barrier is information-theoretic rather than methodological. The
statistics a corrector needs behave like a fast mean-reverting random walk whose
one-step innovation is as large as the drift itself; what the past does not determine
about the present is most of what there is to correct.

\paragraph{Contributions.}
\begin{itemize}\itemsep1pt \parskip0pt
\item \textbf{A formal account of the degradation} (Section~\ref{sec:theorem-a}). The
  masked flow-matching loss decomposes exactly, with no independence assumption, into
  an irreducible entropy plus a divergence (Theorem~\ref{thm:decomp}); along the
  sharpening path training follows, the derivative of the drift-period divergence is
  $(q-p)(\tfrac12-p)/p(1-p)$ (Theorem~\ref{thm:antimono}), positive exactly for slots
  rare in training and common at deployment and divergent as their training probability
  vanishes. At the sampler, drift enters as a step-count-independent error floor
  (Theorem~\ref{thm:sampler}).

\item \textbf{The phenomenon is general} (Section~\ref{sec:universality}). Across seven
  well-powered conditions spanning four domains and four token resolutions, the
  drift-period marginal error is flat in the sampling budget ($\le6.0\%$ over a
  $50\times$ range) while the floor it reaches is $2.2\times$ to $34.3\times$ the
  in-period floor; the elevated floor reproduces on nine domains in total.

\item \textbf{An impossibility result for observation-based correction}
  (Section~\ref{sec:theorem-b}) --- our central contribution. Any corrector measurable
  with respect to past observations leaves at least the conditional variance of the
  target statistic (Theorem~\ref{thm:lb}), and trend extrapolation beats the last
  observation only when $\mu^2>v(1-2\rho)$ (Theorem~\ref{thm:thresh}). We measure both
  premises: the increments are mean-reverting ($\rho\in[-0.29,-0.11]$) and trendless,
  so the threshold is violated and extrapolation is strictly worse everywhere.

\item \textbf{Among realizable corrections, the cheapest is best}
  (Section~\ref{sec:sufficiency}). A training-free marginal anchor beats $200$ gradient
  steps of full test-time adaptation and stacking them is worse than either --- as
  Theorem~\ref{thm:antimono} predicts, since adaptation is training and training on a
  moving target re-sharpens in the wrong place.
\end{itemize}

Each section narrows the space of remedies and closes on the same fact. We regard the
negative result as the useful one: it says where effort should not go --- better
estimators of the observed past --- and where a remedy would have to come from, namely
side information outside the graph stream, or models that decline to sharpen.

%% file: sections/02-related.tex
\section{Related Work and Positioning}
\label{sec:related}

\paragraph{Discrete generative models for graphs.}
Denoising diffusion and flow matching over discrete state spaces are the standard
approach to graph generation: DiGress \citep{vignac2023digress} introduced discrete
denoising diffusion with a marginal-transition prior, DeFoG \citep{qin2024defog}
recast graph generation as discrete flow matching, and generative models of evolving
graphs --- TIGGER \citep{gupta2022tigger}, TagGen \citep{zhou2020taggen}, DYMOND
\citep{zeno2021dymond} --- target realism of the generated stream against held-out
statistics. We study the same family but not the same question: we analyse what
happens to \emph{any} masked or absorbing-state generator when deployment and training
periods differ, so we neither compete on generation quality nor need to. The mechanism
in Section~\ref{sec:theorem-a} follows from the objective's structure and applies to
these models as a class.

\paragraph{Theory of discrete diffusion --- and where it stops.}
Two recent lines bear on our technique. \citet{liang2025discrete} give
sampler-convergence guarantees for discrete diffusion by bounding the \emph{rate of
change} of a KL divergence with a differential inequality rather than the divergence
itself, avoiding change-of-measure arguments and their regularity conditions.
Theorem~\ref{thm:antimono} uses the same move for the same reason: a uniform bound on
the drift-period divergence is unavailable because the relevant log-ratio blows up on
rare slots, whereas the derivative along the sharpening path is well behaved.
\citet{feng2025diffusionlm} establish an efficiency--accuracy dichotomy for masked
diffusion models --- near-optimal token-level error in constant steps, but
sequence-level correctness possibly requiring steps growing with length --- and note
explicitly that their results cover the masked variant of discrete flow matching, so
our models fall in scope. Sharper in-distribution error bounds for masked diffusions
with factorised approximations appear in \citet{lavenant2025errorbounds}.

The gap we occupy is that \emph{all of this analysis is stationary}: it describes
convergence to, and sampling from, the distribution the model trained on.
Theorem~\ref{thm:sampler} is the non-stationary counterpart of the token/sequence
split --- the marginal term becomes a step-independent floor set by the drift, and the
joint term is governed by \emph{target-period} dependence the source-trained model
cannot train away. We do not claim few-step sampling breaks under drift; we measured it
and it does not (Section~\ref{sec:universality}). What changes is the floor it reaches.

\paragraph{Correcting temporal shift, and why we propose no method.}
Temporal domain generalisation anticipates future parameters by extrapolating their
trajectory, as DRAIN \citep{bai2022temporaldg} makes explicit; graph-specific work
addresses out-of-distribution generalisation for evolving graphs
\citep{sun2025evograph}. On the generative side, machinery for steering a discrete
model toward a specified distribution already exists: guidance for discrete
state-space models \citep{nisonoff2025guidance} and primal--dual guided decoding under
constraints \citep{tomasi2026primaldual}. This is exactly why our contribution is a
theorem and not a method: the mechanism for re-aiming a discrete generator is off the
shelf and the idea of extrapolating a temporal trend is already taken, so what is
missing is whether either can work here --- and Section~\ref{sec:theorem-b} answers
that negatively on measured grounds. Publishing the bound and the measurement seems
more useful than adding another corrector to a class we can show is capped.

%% file: sections/03-setup.tex
\section{Setup and Preliminaries}
\label{sec:setup}

\paragraph{Temporal graphs as token sequences.}
A temporal graph is observed as a sequence of snapshots. Around a centre node we
extract, at each snapshot, an ego-network of at most $\Nmax$ nodes and encode it as a
fixed-length binary token $z\in\{0,1\}^{D}$ with $D=\Nmax+\binom{\Nmax}{2}$: the first
$\Nmax$ slots indicate node presence, the rest the edges of the strictly
upper-triangular block. A history $z_{\le t}$ is summarised by an encoder into a
context $h=\mathrm{Enc}_\theta(z_{\le t})$, and the task is to generate the next token
given $h$. Everything below concerns a single such conditional model.

\paragraph{Source and target periods.}
Split the timeline chronologically: $\Rsrc$ is the law of $(z,h)$ on the \emph{source
period} the model trains on, $\Rtgt$ its law on a later \emph{target period} it is
deployed on, and $\Rsrc\ne\Rtgt$ is the drift we study. Write
$p_d(\cdot\mid\Ccal)$, $q_d(\cdot\mid\Ccal)$ for the true conditional law of slot $d$
under $\Rsrc$, $\Rtgt$, and $\phat_d(\cdot\mid\Ccal)$ for the model's. The
\emph{context} $\Ccal=(z^{\sigma(<k)},h)$ is a subset of already-revealed slots plus
the history --- an \emph{any-order} conditional, not a marginal --- a distinction that
is decisive in Sections~\ref{sec:theorem-a} and~\ref{sec:theorem-b}.

\paragraph{Masked discrete flow matching.}
We use the standard masked (absorbing-state) formulation
\citep{ou2024mdm,qin2024defog}: at flow time $s\in[0,1]$ each slot is revealed
independently with probability $s$ and otherwise replaced by $[\mathrm M]$, and the
denoiser predicts the masked slots from the revealed ones and $h$. With $R\subseteq[D]$
the revealed set and $\ell_d(R)=-\log\phat_d(z^d\mid z^R,h)$, training minimises
\begin{equation}
  \Ldfm(\theta)
  \;=\;
  \Ecal_{s\sim U[0,1]}\,\Ecal_{z_s}
  \Big[\tfrac{1}{1-s}\!\!\sum_{d:\,z_s^{(d)}=[\mathrm M]}\!\!\ell_d(R)\Big],
  \label{eq:ldfm}
\end{equation}
with $s$ drawn from $U[0,1-\varepsilon]$ in practice so the weight stays bounded. That
reweighting is what makes \eqref{eq:ldfm} interpretable, via an identity that is not
ours:

\begin{lemma}[any-order autoregressive form; \citealp{ou2024mdm}]
\label{lem:aoar}
With slots revealed independently as above and $s\sim U[0,1]$,
\begin{equation}
  \Ecal_{s}\,\Ecal_{z_s}
  \Big[\tfrac{1}{1-s}\!\!\sum_{d:\,z_s^{(d)}=[\mathrm M]}\!\!\ell_d(R)\Big]
  \;=\;
  \Ecal_{\sigma\sim\mathrm{Unif}(S_D)}\sum_{k=1}^{D}\ell_{\sigma(k)}\big(\sigma(<k)\big).
  \label{eq:aoar}
\end{equation}
\end{lemma}

\noindent
The masked objective is thus exactly an autoregressive log-loss averaged over all $D!$
orderings. Every result here derives from the right-hand side, so we verify the
identity numerically on our own token family in Appendix~\ref{app:lemma0}.

\paragraph{Standing assumptions.}
We assume only that predicted and true conditionals are bounded away from the simplex
boundary, $\phat_d,p_d,q_d\in[\varepsilon,1-\varepsilon]$ --- automatic for a softmax
denoiser and for Laplace-smoothed empirical frequencies. No independence assumption
across slots is made anywhere.

%% file: sections/04-theorem-a.tex
\section{The Sharpening--Drift Tension}
\label{sec:theorem-a}

We show that the training loss decomposes exactly into an irreducible entropy and a
divergence (Theorem~\ref{thm:decomp}), that continuing to fit the source period
eventually \emph{increases} the divergence on the target period
(Theorem~\ref{thm:antimono}), and that at the sampler the drift enters as a
step-count-independent error floor (Theorem~\ref{thm:sampler}).

\begin{theorem}[exact decomposition]
\label{thm:decomp}
For any evaluation law $R\in\{\Rsrc,\Rtgt\}$,
\begin{equation}
\begin{aligned}
  \Ldfm(\theta;R)
  &\;=\;
  \underbrace{\Hent_R(z\mid h)}_{\text{model-independent}}
  \;+\;
  \underbrace{\Delta(\theta;R)}_{\ge 0},
  \\[2pt]
  \Delta(\theta;R)
  &\;=\;\Ecal_{\sigma}\sum_{k=1}^{D}\Ecal_{\Ccal\sim R}\,
  \KL\!\big(r_{\sigma(k)}(\cdot\mid\Ccal)\,\|\,\phat_{\sigma(k)}(\cdot\mid\Ccal)\big).
\end{aligned}
  \label{eq:decomp}
\end{equation}
\end{theorem}

\begin{proof}[Proof sketch]
Rewrite $\Ldfm$ in any-order form with Lemma~\ref{lem:aoar}; for fixed $\sigma,k$ the
inner expectation splits as entropy plus KL, and the entropy terms telescope by the
chain rule to $\Hent_R(z\mid h)$ \emph{for every fixed $\sigma$}, so $\Ecal_\sigma$
drops. Appendix~\ref{app:proofs}.
\end{proof}

Equation~\eqref{eq:decomp} uses no assumption beyond Lemma~\ref{lem:aoar}, in
particular no independence across slots. Two consequences follow. All of the model's
influence on the target-period loss sits in $\Delta(\theta;\Rtgt)$; and in the
nonparametric limit a model trained to optimality on the source period has
$\phat_d\to p_d$, so
\begin{equation}
  \Ldfm(\theta^\ast;\Rtgt)\;\longrightarrow\;
  \Hent_{\Rtgt}(z\mid h)+\Ddrift,
  \qquad
  \Ddrift:=\Ecal_\sigma\sum_{k}\Ecal_{\Ccal\sim \Rtgt}\,
  \KL\!\big(q_{\sigma(k)}\|p_{\sigma(k)}\big).
  \label{eq:floor}
\end{equation}
Fitting the source period perfectly converges not to the best achievable target-period
loss but to a strictly positive floor set by $\Ddrift$.

\subsection{Sharpening is anti-monotone in the drift period}

That a floor exists does not yet say training harder makes things worse. For that we
parameterise the training trajectory as a \emph{sharpening path} from the
uninformative prediction to the perfect source fit,
\begin{equation}
  m_d^{(\alpha)}(\cdot\mid\Ccal)\;=\;(1-\alpha)\cdot\tfrac12+\alpha\,p_d(\cdot\mid\Ccal),
  \qquad \alpha\in[0,1].
  \label{eq:path}
\end{equation}

\begin{theorem}[anti-monotonicity of sharpening]
\label{thm:antimono}
Write $p=p_d(1\mid\Ccal)$ and $q=q_d(1\mid\Ccal)$ for a Bernoulli slot. Then
\begin{equation}
  \frac{\mathrm{d}}{\mathrm{d}\alpha}\,\KL\!\big(q\,\big\|\,m^{(\alpha)}\big)
  \bigg|_{\alpha=1}
  \;=\;
  \frac{(q-p)\big(\tfrac12-p\big)}{p\,(1-p)}.
  \label{eq:antimono}
\end{equation}
Consequently the derivative is strictly positive iff
$\mathrm{sign}(q-p)=\mathrm{sign}(\tfrac12-p)$, and for fixed $q$ it diverges as
$q/(2p)$ as $p\to0$.
\end{theorem}

\begin{proof}[Proof sketch]
$m'(\alpha)=p-\tfrac12$ is constant along \eqref{eq:path}; differentiate the Bernoulli
cross-entropy in $m$ and evaluate at $\alpha=1$. Appendix~\ref{app:proofs}.
\end{proof}

\paragraph{Reading the statement.}
Theorem~\ref{thm:antimono} is a \emph{differential} property, not a uniform lower
bound, deliberately: a uniform bound is unavailable because $\|\log(p/\phat)\|_\infty$
blows up precisely on the rare slots that matter. Differentiating along a path avoids
the blow-up while still answering the question --- \emph{does the next increment of
source-period fit help or hurt the target period?} This is the discrete-diffusion
analogue of bounding a KL divergence's rate of change rather than the divergence
itself, the technique by which the sampler-convergence literature avoids regularity
conditions \citep{liang2025discrete}.

The sign condition reads concretely: slots \emph{rare in the source period and common
in the target period} are where further sharpening hurts, unboundedly so as the source
probability vanishes. In our data the highest-drift slots all satisfy it --- one edge
slot moves from $p=0.0014$ to $q=0.185$, a derivative of $65.6$; another from
$p=0.042$ to $q=0.605$, derivative $6.4$. The model learns these as near-impossible
and deployment then makes them frequent: it is confidently, expensively wrong.

Figure~\ref{fig:tension}(a) shows this on a single training trajectory at five
checkpoints with three seeds. Past roughly $500$ steps every further reduction in
source-period loss buys an increase in target-period loss, and the trade is a clean
power law, $\mathcal L_{\Rtgt}\propto\mathcal L_{\Rsrc}^{-0.605}$ with $R^2=0.9977$ on
the seed means. All three seeds are pointwise monotone (Spearman $=-1.000$ each), so
this is no averaging artefact: halving the in-period loss multiplies the drift-period
loss by about $1.52$.

\begin{figure}[t]
  \centering
  \includegraphics{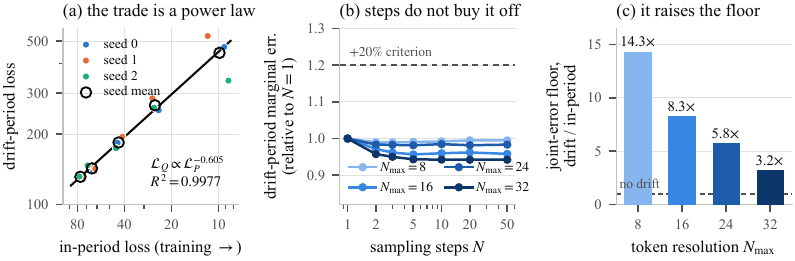}
  \caption{The sharpening--drift tension.
  \textbf{(a)} The trade is a power law: each point is one checkpoint
  ($250$--$4000$ steps) of one seed, the axis inverted so training progresses left to
  right; the fit is on the seed means (open circles) and all three seeds are
  individually monotone.
  \textbf{(b)} Drift-period marginal error against sampling steps, normalised to
  $N=1$, for four token resolutions on Enron; no curve comes within $6\%$ of the
  $\pm20\%$ flatness criterion, whose lower edge is therefore omitted.
  \textbf{(c)} The floor that sampling converges to, drift period relative to
  in-period, for the same four models.}
  \label{fig:tension}
\end{figure}

\subsection{At the sampler: drift enters the floor, not the step count}

Theorems~\ref{thm:decomp} and~\ref{thm:antimono} concern the likelihood, but a model is
used through a sampler revealing slots over $N$ steps, so the practical question is
whether the drift cost can be bought off with sampling compute. It cannot.

\begin{theorem}[sampler-level decomposition]
\label{thm:sampler}
Let the sampler reveal disjoint slot blocks $A_0,\dots,A_{N-1}$ and let
$P_{\rm gen}$ be the induced generative law. Then
\begin{equation}
  \Ddrift
  \;\le\;
  \KL\big(P_{\rm gen}\,\|\,\Rtgt\big)
  \;\le\;
  \Ddrift
  \;+\;
  \sum_{k=0}^{N-1}\Ecal_{\Ccal_k}\,\IQ\big(z^{A_k}\mid\Ccal_k\big),
  \label{eq:sampler}
\end{equation}
where $\IQ(z^{A}\mid\Ccal)=\KL\big(q(z^{A}\mid\Ccal)\,\|\,\prod_{d\in A}q_d(\cdot\mid\Ccal)\big)\ge0$
is the \emph{target-period} conditional total correlation within a block. The first
term does not depend on $N$; the second is non-increasing under refinement of the
schedule and vanishes at $N=D$.
\end{theorem}

\begin{proof}[Proof sketch]
Multiply and divide by $\prod_{d\in A}q_d$ inside the log to split each block's KL
into a total-correlation term plus per-slot drift KLs; substitute into the chain
decomposition and average over reveal orders. Each slot is revealed exactly once, so
the drift terms sum to $\Ddrift$ whatever the schedule; monotonicity in $N$ is the
chain rule for multi-information. Appendix~\ref{app:proofs}.
\end{proof}

The split in \eqref{eq:sampler} is a non-stationary counterpart of the
token-error/sequence-error dichotomy for masked diffusion models
\citep{feng2025diffusionlm}, with a sharp practical consequence.

\begin{corollary}[step-independent floor]
\label{cor:floor}
Any per-slot marginal functional of the generated distribution --- density, degree, and
the marginal-matching statistics generative-graph work reports --- sees only the
$\Ddrift$ term of \eqref{eq:sampler}. Its error floor is therefore non-decreasing in
$N$: additional sampling steps cannot improve marginal agreement with the target
period.
\end{corollary}

Figure~\ref{fig:tension}(b) measures this: sweeping the sampler from $N=1$ to $N=50$
moves the drift-period marginal error by at most $6.0\%$ of its value, at every token
resolution. The floor itself is large --- Figure~\ref{fig:tension}(c) puts the
drift-period joint-error floor $3.2\times$ to $14.3\times$ above the in-period floor
on the same model.

One might expect drift to make few-step sampling untenable by inflating the second
term of \eqref{eq:sampler}. On these data it does not, and we return to that in
Section~\ref{sec:discussion}: \emph{few-step samplers survive drift; what they
converge to is simply worse.}

%% file: sections/05-universality.tex
\section{Universality of the Phenomenon}
\label{sec:universality}

Section~\ref{sec:theorem-a} derived the tension from an identity that assumes almost
nothing, so it should not be a property of one dataset or resolution. This section
reports how far it reaches and --- since the answer is not uniform --- where it weakens.

We sweep two axes: token resolution $\Nmax\in\{8,16,24,32\}$ on Enron, and three
further domains at $\Nmax=16$, each with $\approx1500$ held-out sequences
(Table~\ref{tab:universality} lists every condition). Two things hold without
exception. The drift-period marginal error is flat in the sampling budget --- over
$N\in[1,50]$ it moves by at most $6.0\%$ of its own value, confirming
Corollary~\ref{cor:floor} in $7/7$ conditions --- and the floor sampling converges to is
$2.2\times$ to $34.3\times$ higher in the drift period than in-period. The drift cost is
entirely a matter of \emph{where} sampling converges, not \emph{how fast}.

Five further domains (Contacts, Flights, SocialEvo, UNtrade, LastFM) carry only around
$75$ held-out sequences, enough to resolve a large effect only. On that effect they
agree: drift-period floors run $4.0\times$ to $47.6\times$ their in-period
counterparts. UNvote is excluded as degenerate --- its held-out split is identically
zero on the top-active slots, so the ratio is undefined rather than large. Counting
these, the elevated floor reproduces across nine distinct domains and twelve
conditions.

The Enron floor ratio \emph{falls} as $\Nmax$ grows,
$14.3\times\!\rightarrow\!3.2\times$. We flag this because the opposite is a tempting
inference from \eqref{eq:decomp}: $\Ddrift$ sums over $D$ slots, so one might expect
the penalty to accumulate with dimension. It does not, and
Theorem~\ref{thm:antimono} explains why --- the per-slot penalty is set by how far each
coordinate's distribution moved, not by how many coordinates there are, and raising
$\Nmax$ mostly adds sparse high-order edge slots that dilute the average. The drift
cost is a property of the shift, not of the dimension.

\paragraph{Where the drift actually lives.}
Which part of the distribution the damage occupies matters, because
Section~\ref{sec:theorem-b} turns on it. On Enron at $\Nmax=32$ the model drives its
in-period loss to $9.85$ nats, roughly a ninth of the $88.72$ that per-slot marginal
entropies alone would cost, the compression coming from inter-slot dependence and
history conditioning. In the drift period that same conditional structure costs it:
the loss reaches $445.77$ nats, about $2.8\times$ the $156.68$ that per-slot marginals
would predict, while measured drift in the unconditional marginals amounts to only
$42.2$ nats --- a small fraction of the $436$-nat gap between in-period and
drift-period loss. The drift is overwhelmingly a shift in \emph{conditional}
structure; the unconditional marginals --- the only part deployment-time observation
can readily estimate --- are a small share of it, and Section~\ref{sec:theorem-b}
shows that even that share cannot be tracked.

\paragraph{A second model family, and an honest limit.}
Section~\ref{sec:theorem-a} applies to any masked or absorbing-state generative model,
so we probed a second family: DiGress \citep{vignac2023digress}, a discrete denoising
diffusion model with a marginal-transition prior, trained on the same tokens for $32$k
steps. The direction reproduces --- its in-period loss falls from $187.4$ to $145.3$
nats while its drift-period loss stays flat ($312.8$ to $302.0$), so the
drift-to-in-period ratio climbs monotonically from $1.67$ to $2.08$ and no in-period
progress buys any drift-period progress.

The power law does not reproduce, however, and our pre-registered criterion for this
experiment returned FAIL ($R^2=0.30$). The likely reason is itself consistent with
Theorem~\ref{thm:antimono}: DiGress barely sharpens. Its in-period loss improves by a
factor of $1.29$ over $32$k steps where ours improves by $7.8$ over $4$k, its
marginal-transition prior braking exactly the confidence growth \eqref{eq:antimono}
makes costly. A model that does not sharpen should show a muted tension, and it does.
We report this as directional corroboration only --- one seed and a failed criterion,
not a second power law --- and note it suggests the sharpening rate, not the model
family, is the operative variable.

%% file: sections/06-theorem-b.tex
\section{Impossibility of Observation-Based Correction}
\label{sec:theorem-b}

The obvious response to Sections~\ref{sec:theorem-a} and~\ref{sec:universality} is to
correct the drift: at deployment we do see target-period snapshots, unlabelled but in
full, so why not measure how the distribution moved and re-aim the generator? This
section shows why that cannot work. The reason is not that the correctors we tried
were too weak; it is that the required information is not in the observations.

\paragraph{Setting.}
Fix a slot $d$ and let $q_t:=\Prob_{\Rtgt}(z_d=1\mid \text{time } t)$ be its
target-period marginal, a stochastic process adapted to the observation filtration
$\Fcal_t=\sigma(\text{target snapshots observed through } t)$. The source-trained
generator emits a fixed $p_d$. An \emph{observation-based corrector} forms $\qhat_t$
from strictly past observations --- formally, $\qhat_t$ is $\Fcal_{t-1}$-measurable ---
and re-aims the sampler's per-slot marginal onto it; the unattainable \emph{oracle}
sets $\qhat_t=q_t$. By Theorem~\ref{thm:decomp} the quantity such a correction moves
is $\KL(q_t\|\qhat_t)$.

\begin{theorem}[optimal observation-based correction leaves the conditional variance]
\label{thm:lb}
For any $\Fcal_{t-1}$-measurable $\qhat_t$,
$\Ecal[(q_t-\qhat_t)^2]\ge\Ecal[\Var(q_t\mid\Fcal_{t-1})]$, with equality iff
$\qhat_t=\Ecal[q_t\mid\Fcal_{t-1}]$ a.s. Consequently
\begin{equation}
  \Ecal\big[\KL(q_t\|\qhat_t)\big]\;\ge\;2\,\Ecal\big[\Var(q_t\mid\Fcal_{t-1})\big],
  \label{eq:kllb}
\end{equation}
while the oracle attains $0$.
\end{theorem}

\begin{proof}[Proof sketch]
Expand around $m_t=\Ecal[q_t\mid\Fcal_{t-1}]$; the cross term vanishes because
$m_t-\qhat_t$ is $\Fcal_{t-1}$-measurable and $\Ecal[q_t-m_t\mid\Fcal_{t-1}]=0$, so
$\Ecal[(q_t-\qhat_t)^2]=\Ecal[\Var(q_t\mid\Fcal_{t-1})]+\Ecal[(m_t-\qhat_t)^2]$.
Pinsker on a Bernoulli slot gives \eqref{eq:kllb}. Appendix~\ref{app:proofs}.
\end{proof}

Theorem~\ref{thm:lb} is deliberately unglamorous --- it is the $L^2$ optimality of
conditional expectation --- and that is the point. The residual is a property of the
process, not of the method: whatever the past does not determine about $q_t$ survives
\emph{every} observation-based correction, and only the oracle reaches zero. No
distributional assumption is used.

Two natural correctors make the question measurable: \emph{persistence}
$\qhat_t^{P}=q_{t-1}$, and \emph{extrapolation} $\qhat_t^{E}=2q_{t-1}-q_{t-2}$. Write
$\Delta_t=q_t-q_{t-1}$, $\mu=\Ecal[\Delta_t]$, $v=\Var(\Delta_t)$ and
$\rho=\mathrm{Corr}(\Delta_t,\Delta_{t-1})$.

\begin{theorem}[exact threshold for extrapolation]
\label{thm:thresh}
If the increment process is second-order stationary, then
$\mathrm{MSE}(P)=v+\mu^2$ and $\mathrm{MSE}(E)=2v(1-\rho)$, so
\begin{equation}
  \text{extrapolation beats persistence}\iff \mu^2>v\,(1-2\rho).
  \label{eq:thresh}
\end{equation}
\end{theorem}

\begin{proof}[Proof sketch]
The persistence error is $\Delta_t$ and the extrapolation error is
$\Delta_t-\Delta_{t-1}$, of mean zero and variance $2v(1-\rho)$; compare.
Appendix~\ref{app:proofs}.
\end{proof}

Two special cases turn \eqref{eq:thresh} into a protocol. Under a \emph{martingale}
($\mu=\rho=0$), persistence is exactly the optimal predictor of Theorem~\ref{thm:lb},
its residual is the one-step innovation variance, and extrapolation is exactly
$2\times$ worse. Under \emph{mean reversion} ($\rho<0$, $\mu\approx0$),
$\mathrm{MSE}(E)=2v(1-\rho)>2v$: worse still, and worse as $|\rho|$ grows. Only real
momentum, $\rho>1/2$, or a strong trend flips \eqref{eq:thresh}. We confirmed both
theorems' algebra against a $2\times10^5$-step Monte-Carlo simulation.

\begin{figure}[t]
  \centering
  \includegraphics{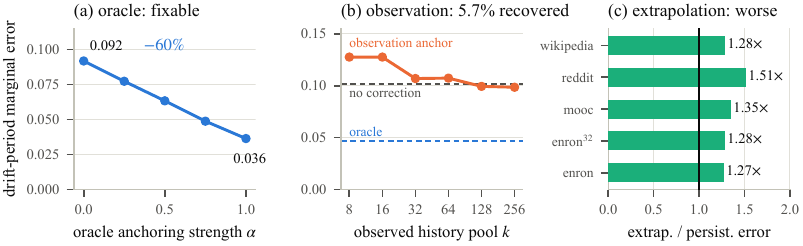}
  \caption{The impossibility, measured. \textbf{(a)} An oracle knowing the current
  target marginal removes $60\%$ of the drift-period marginal error --- the quantity
  \emph{is} correctable in principle. \textbf{(b)} A corrector restricted to observed
  history recovers $5.7\%$ of that, and enlarging the history pool $k$ does not close
  the gap. \textbf{(c)} Extrapolating the observed trend is strictly worse than
  trusting the last observation, in every condition. Panels (a) and (b) are separate
  runs starting from slightly different uncorrected errors ($0.092$, $0.102$); each
  carries its own reference levels.}
  \label{fig:impossibility}
\end{figure}

\paragraph{The oracle attains the bound (Fig.~\ref{fig:impossibility}a).}
Anchoring per-slot marginals to the \emph{true} target values, with strength $\alpha$
swept from $0$ to $1$, drops the drift-period marginal error from $0.092$ to $0.037$
($-60\%$), and the joint error falls by $50\%$ rather than degrading, so the correction
does not trade marginal accuracy against structure. The target of Theorem~\ref{thm:lb}
is reachable when $\qhat_t=q_t$.

\paragraph{Observation-based correction hits the bound (Fig.~\ref{fig:impossibility}b).}
Estimating the same marginals from the $k$ most recent observed snapshots recovers only
$5.7\%$ of the oracle's gain at $k=256$, and the curve is flat in $k$: more history does
not help. The magnitude checks out against Theorem~\ref{thm:lb} --- the one-step
innovation, measured as the mean absolute change in per-slot marginals between adjacent
windows, is $0.069$ against a total uncorrected error of $0.102$, so the unpredictable
part is most of the drift and \eqref{eq:kllb} leaves almost nothing to recover.
\emph{The bound is tight and the corrector is sitting on it.}

\paragraph{The premise of \eqref{eq:thresh} is measurable, and adverse
(Fig.~\ref{fig:impossibility}c).}
We measured its inputs directly from token marginals, with no model in the loop
(per-condition values in Table~\ref{tab:directionality}). Across five conditions the
increments are mean-reverting, $\rho\in[-0.29,-0.11]$, and trendless: consecutive
increments keep their sign only $0.12$ to $0.28$ of the time, far below the $0.5$ of a
coin flip, so $\mu\approx0$. The right-hand side of \eqref{eq:thresh} then exceeds $v$
while the left-hand side is $\approx0$: the condition is violated and extrapolation
must be strictly worse. It is, in all five conditions, by $1.27\times$ to $1.51\times$.
Theorem~\ref{thm:thresh} predicts $2(1-\rho)\in[2.2,2.6]$ in squared error, roughly
$1.5$--$1.6\times$ in absolute error; the measurement sits at or just below that, the
gap attributable to clipping extrapolated values into $[0,1]$ (which asymmetrically
helps extrapolation), the $L^1$ versus $L^2$ conversion, and finite windows. The robust
prediction --- extrapolation is worse, and worse as $-\rho$ grows --- holds without
exception and across window counts (Appendix~\ref{app:tsweep}).

\paragraph{Together.}
The oracle reaches the top of the range; the best observation-based corrector is pinned
near the bottom by Theorem~\ref{thm:lb}; and the one manoeuvre that could have escaped
the bound --- looking ahead by extrapolating a trend --- is ruled out by
Theorem~\ref{thm:thresh} because the trend does not exist. All three close on the same
fact: \emph{the drift is fast, mean-reverting, high-dimensional noise}, whose one-step
innovation is the size of the drift itself, which simultaneously makes
\eqref{eq:kllb} nearly maximal and \eqref{eq:thresh} violated.

\paragraph{Scope, stated precisely.}
Theorem~\ref{thm:lb} lower-bounds the \emph{marginal} part only. What a corrector
estimates is the \emph{unconditional} per-slot marginal, whereas the drift in
Theorem~\ref{thm:sampler} also lives in conditional structure such a correction leaves
untouched. This makes \eqref{eq:kllb} a valid lower bound on the \emph{total} residual
--- an inequality, not an identity --- and strengthens the conclusion: even the most
easily estimated component cannot be tracked. We never slide between the two senses of
``marginal''. Two further limits, that Theorem~\ref{thm:thresh} assumes second-order
stationarity and that we measure $\rho$ pooled rather than per-slot, are stated in
Section~\ref{sec:discussion}.

This is therefore an impossibility \emph{conditional on measurable properties of the
drift}. Theorem~\ref{thm:thresh} is satisfiable --- momentum $\rho>1/2$ would flip it
--- and Theorem~\ref{thm:lb} is vacuous at zero conditional variance, correctly
recovering the stationary case. The content is that real temporal graph streams land
on the wrong side of both.

%% file: sections/07-sufficiency.tex
\section{Minimal Sufficiency Among Realizable Corrections}
\label{sec:sufficiency}

Section~\ref{sec:theorem-b} bounds what \emph{any} observation-based corrector can do.
A natural objection is that re-aiming per-slot marginals is the weakest thing one could
try and that real compute would do better. It does not, and spending more can hurt. We
compare four arms on the same drift period: \emph{no correction}, the source-trained
generator as-is; \emph{marginal anchor}, the corrector of Section~\ref{sec:theorem-b},
which estimates per-slot marginals from observed history and re-aims the sampler at
zero gradient steps; \emph{full TTA}, which instead adapts \emph{all} parameters on
unlabelled target-period data for $200$ gradient steps; and \emph{TTA $+$ anchor}.

Against an uncorrected drift-period marginal error of $0.102$, the training-free
anchor reaches $0.088$ while $200$ steps of full test-time adaptation reach only
$0.097$: the anchor attains $0.913\times$ the error of TTA at none of its cost.
Stacking is worse still, at $0.098$ --- $1.112\times$ the anchor's error --- and the
joint error shows the same ordering more sharply ($0.413$ for the anchor against
$0.476$ stacked, a $15\%$ degradation).

Theorem~\ref{thm:antimono} accounts for the reversal. Test-time adaptation is
training, and training on the small pool of observed target data sharpens the model on
that pool --- precisely the operation \eqref{eq:antimono} penalises when the
deployment distribution keeps moving, which by Section~\ref{sec:theorem-b} it does,
fast and mean-revertingly. The adapted model is better aimed at where the target
period \emph{was} during the adaptation window and worse aimed at where it goes next;
adding the anchor on top inherits that mis-aim rather than undoing it.

This answers the practical form of the question Theorem~\ref{thm:lb} leaves open.
Within the realizable class, the minimal correction is at least as good as
full-parameter adaptation costing $200$ gradient steps, and combining them is
counterproductive. Every arm lands in the same narrow band, well short of the $0.047$
an oracle reaches from the same $0.102$ starting point
(Figure~\ref{fig:impossibility}b) --- which is what a binding bound looks like from
the inside.

%% file: sections/08-discussion.tex
\section{Discussion and Limitations}
\label{sec:discussion}

\paragraph{Scope of the evidence.}
Theorems~\ref{thm:decomp}--\ref{thm:sampler} and~\ref{thm:lb}--\ref{thm:thresh} follow
from the structure of the masked objective and the measurability of the corrector, and
refer to no particular tokenisation, encoder, or model family; the experiments,
however, run predominantly on one carrier, and the second-family probe
(Section~\ref{sec:universality}) is a single seed that missed its pre-registered
criterion. The empirical claim covers resolutions and domains, not architectures. Two
limits deferred from Section~\ref{sec:theorem-b} also belong here:
Theorem~\ref{thm:thresh} assumes second-order stationarity, true only approximately in
finite windows, so we treat its ratio as a same-order comparison and rest the argument
on the sign; and we measure $\rho$ pooled over active slots, inferring $\mu\approx0$
from sign persistence, whereas \eqref{eq:thresh} is per-slot.

\paragraph{The oracle is a yardstick, not a method.}
Figure~\ref{fig:impossibility}(a) is not a proposal. Its role is to show that the
$60\%$ no observation-based corrector reaches \emph{is} reachable with the right
information; without it the flatness in Figure~\ref{fig:impossibility}(b) would be
equally consistent with the error being irreducible, and the impossibility would have
no teeth.

\paragraph{Home-grown error metrics.}
The marginal and pairwise-joint errors are our own constructions
(Appendix~\ref{app:metrics}), so we state claims as relative magnitudes --- a floor
ratio, a fraction of the oracle's gain recovered, an extrapolation-to-persistence ratio
--- never as absolute scores comparable across papers. Two earlier versions of our
evaluation carried outright defects we found and fixed: a spectral statistic constant
by construction, and an Erd\H{o}s--R\'enyi control matched so tightly the metric it fed
was mathematically unsatisfiable. That history is the reason for the policy.

\paragraph{Measurement window, and criteria we amended.}
Appendix~\ref{app:tsweep} sweeps the number of chronological windows used to estimate
the drift process, and Appendix~\ref{app:protocol} records three pre-registered criteria
that had design defects and which we amended after seeing results. What carries
Theorem~\ref{thm:thresh} is robust across the sweep: increment autocorrelation never
approaches the $\rho>1/2$ the threshold requires (maximum $+0.22$), sign persistence
never approaches chance (maximum $0.33$), and extrapolation never beats persistence on
the deployment period outside the coarsest binning. But at the finest binning ($T=50$)
the pre-registered directionality flag flips in four of five conditions --- on the
full-span axis, contradicting the deployment-period measurement in the same conditions,
where extrapolation is $1.24$--$1.53\times$ worse. We read that as a binning artefact
at window counts finer than the tokenised blocks' resolution, and print the whole sweep
so a reader can weigh it.

\paragraph{A conjecture of ours that the data refuted.}
We predicted from Theorem~\ref{thm:sampler} that drift would inflate the target-period
dependence $\IQ$ and make few-step sampling untenable under shift --- the
non-stationary analogue of the sequence-level hardness of \citet{feng2025diffusionlm}.
It does not: in all seven well-powered conditions the drift-period joint error is no
more step-sensitive than the in-period one. We state this as a refuted prediction
rather than quietly narrowing the theorem, and the resulting picture is friendlier:
few-step samplers are safe under drift, only their destination is worse.

\paragraph{What would actually help.}
Two directions survive our own analysis. First, side information not in the observed
stream --- exogenous covariates, calendars, announcements --- since
Theorem~\ref{thm:lb} bounds only stream-measurable correctors, and anything genuinely
predictive of $q_t$ enlarges the filtration and lowers the bound. Second, not
sharpening: \eqref{eq:antimono} makes the penalty a function of how confident the model
becomes on slots rare in training. Mechanisms for the latter --- entropy floors,
confidence penalties, guidance and constrained decoding
\citep{nisonoff2025guidance,tomasi2026primaldual} --- are established technique and we
claim no credit for them. Our contribution is the account of why they, and not better
estimators of the observed past, are where the headroom is.

%% file: sections/09-conclusion.tex
\section{Conclusion}
\label{sec:conclusion}

Temporal graph generators degrade on the period after the one they trained on, in a
derivable way. The masked flow-matching loss splits exactly into an entropy and a
divergence, and along the path training actually takes, the divergence on a later
period has derivative $(q-p)(\tfrac12-p)/p(1-p)$ --- positive precisely for structures
that were rare during training and become common at deployment, unbounded as their
training probability approaches zero. Empirically the trade is a power law with
exponent $-0.605$, it holds across four domains and four token resolutions, and at the
sampler it raises the error floor without changing how many steps reach it.

The more useful half is what happens when one tries to fix this. Because the
deployment period is observed, correction looks like a matter of measurement. It is
not. Any corrector measurable with respect to past observations leaves at least the
conditional variance of the quantity it tracks, and trend extrapolation helps only
when $\mu^2>v(1-2\rho)$. Both quantities are measurable and both go the wrong way
here: the drift is trendless and its one-step innovation is as large as the drift
itself. An oracle removes $60\%$ of the error; the best observation-based corrector
recovers $5.7\%$ of that; extrapolation is strictly worse than doing nothing clever;
and $200$ gradient steps of test-time adaptation buy less than a training-free anchor
and, stacked, buy negative.

This narrows the search rather than closing it. Effort spent on better estimators of
the observed past is capped by a bound we can write down, so it should go elsewhere:
to information from outside the graph stream, or to models that decline to sharpen. For
theory it marks a boundary. The convergence and efficiency results underpinning
discrete diffusion and flow matching are stationary and survive here --- few-step
sampling still works under drift. What does not survive is the assumption that the
distribution being sampled is the one that was trained on, and that assumption turns
out to carry an impossibility with it.

%% file: sections/10-appendix.tex
\appendix

\section{Numerical verification of Lemma~\ref{lem:aoar}}
\label{app:lemma0}

Lemma~\ref{lem:aoar} is an exact identity, but it is the load-bearing step for every
result here, so we checked it on our own token family rather than trusting the
restatement. At $D=4$ the right-hand side of \eqref{eq:aoar} can be enumerated
exactly over all $4!=24$ orderings; the left-hand side we estimated with
$2\times10^5$ Monte-Carlo draws of $(s,z_s)$. The two agree to a relative error below
$2\%$, consistent with the Monte-Carlo standard error.

\section{Proofs}
\label{app:proofs}

\subsection{Proof of Theorem~\ref{thm:decomp}}

By Lemma~\ref{lem:aoar}, $\Ldfm(\theta;R)=\Ecal_\sigma\sum_{k=1}^{D}
\Ecal_{\Ccal\sim R}\,\Ecal_{z^{\sigma(k)}\sim r(\cdot\mid\Ccal)}
\big[-\log\phat_{\sigma(k)}(\cdot\mid\Ccal)\big]$ with
$\Ccal=(z^{\sigma(<k)},h)$. For a fixed $\sigma$, $k$ and $\Ccal$, the inner
expectation is a cross-entropy and splits as
\begin{equation*}
  \Ecal_{z^{\sigma(k)}\sim r(\cdot\mid\Ccal)}\big[-\log\phat(\cdot\mid\Ccal)\big]
  \;=\;
  \Hent\big(r(\cdot\mid\Ccal)\big)
  \;+\;
  \KL\big(r(\cdot\mid\Ccal)\,\|\,\phat(\cdot\mid\Ccal)\big).
\end{equation*}
Take expectations over $\Ccal\sim R$ and sum over $k$. The KL terms are by definition
$\Delta(\theta;R)$, and each is non-negative. For the entropy terms, the chain rule
gives $\sum_{k}\Hent\big(r_{\sigma(k)}\mid \sigma(<k)\big)=\Hent_R(z\mid h)$ for
\emph{every} fixed permutation $\sigma$; since this value does not depend on $\sigma$,
the outer expectation $\Ecal_\sigma$ acts trivially on it and can be removed.
\hfill$\square$

\subsection{Proof of Theorem~\ref{thm:antimono}}

Along the path \eqref{eq:path}, $m(\alpha)=\tfrac{1-\alpha}{2}+\alpha p$, so
$m'(\alpha)=p-\tfrac12$ independently of $\alpha$. For a Bernoulli slot,
\begin{equation*}
  \frac{\partial}{\partial m}\Big[-q\log m-(1-q)\log(1-m)\Big]
  \;=\;
  -\frac{q}{m}+\frac{1-q}{1-m},
\end{equation*}
and $\KL(q\|m)$ differs from this cross-entropy by a term independent of $m$. By the
chain rule, evaluating at $\alpha=1$ (where $m=p$),
\begin{equation*}
  \frac{\mathrm d}{\mathrm d\alpha}\KL\big(q\|m^{(\alpha)}\big)\bigg|_{\alpha=1}
  =\Big[-\frac{q}{p}+\frac{1-q}{1-p}\Big]\Big(p-\tfrac12\Big)
  =\frac{-q(1-p)+p(1-q)}{p(1-p)}\Big(p-\tfrac12\Big),
\end{equation*}
which simplifies to $(q-p)\big(\tfrac12-p\big)\big/\big(p(1-p)\big)$.
The sign claim is immediate. For fixed $q>0$, as $p\to0$ the expression behaves as
$q\cdot\tfrac12/p=q/(2p)\to\infty$. \hfill$\square$

\subsection{Proof of Theorem~\ref{thm:sampler}}

Fix a block $A$ revealed jointly given context $\Ccal$. A sampler that reveals the
slots of $A$ in one step emits them independently given $\Ccal$, with per-slot
probabilities $p_d(\cdot\mid\Ccal)$. Multiplying and dividing by
$\prod_{d\in A}q_d(z^d\mid\Ccal)$ inside the logarithm,
\begin{equation*}
  \Ecal_{q}\log\frac{q(z^{A}\mid\Ccal)}{\prod_{d\in A} p_d}
  =\Ecal_{q}\log\frac{q(z^{A}\mid\Ccal)}{\prod_{d\in A} q_d}
  +\sum_{d\in A}\Ecal_{q}\log\frac{q_d}{p_d}
  =\IQ\big(z^{A}\mid\Ccal\big)+\sum_{d\in A}\KL\big(q_d\|p_d\big).
\end{equation*}
Substituting this into the chain decomposition over blocks and averaging over reveal
orders with Lemma~\ref{lem:aoar}: each slot belongs to exactly one block, so the drift
terms aggregate to $\sum_{d=1}^{D}\Ecal_{\Ccal\sim \Rtgt}\KL(q_d\|p_d)=\Ddrift$,
independently of the schedule. This gives the upper bound in \eqref{eq:sampler}; the
lower bound follows since $\IQ\ge0$ and the drift terms are present for any schedule.
Monotonicity under refinement is the chain rule for multi-information: splitting a
block into sub-blocks cannot increase the total correlation charged, and at $N=D$ every
block is a singleton with $\IQ=0$. \hfill$\square$

\subsection{Proof of Theorem~\ref{thm:lb}}

Let $m_t=\Ecal[q_t\mid\Fcal_{t-1}]$ and expand
\begin{equation*}
  \Ecal[(q_t-\qhat_t)^2]
  =\Ecal[(q_t-m_t)^2]+2\,\Ecal[(q_t-m_t)(m_t-\qhat_t)]+\Ecal[(m_t-\qhat_t)^2].
\end{equation*}
The cross term vanishes: $m_t-\qhat_t$ is $\Fcal_{t-1}$-measurable and
$\Ecal[q_t-m_t\mid\Fcal_{t-1}]=0$, so conditioning on $\Fcal_{t-1}$ and applying the
tower property gives zero. Hence
$\Ecal[(q_t-\qhat_t)^2]=\Ecal[\Var(q_t\mid\Fcal_{t-1})]+\Ecal[(m_t-\qhat_t)^2]
\ge\Ecal[\Var(q_t\mid\Fcal_{t-1})]$, with equality iff $\qhat_t=m_t$ almost surely.
For the KL form, Pinsker's inequality on a Bernoulli slot gives
$\KL(q_t\|\qhat_t)\ge2\,\TV^2=2(q_t-\qhat_t)^2$; taking expectations yields
\eqref{eq:kllb}. The oracle $\qhat_t=q_t$ gives $\KL=0$. \hfill$\square$

\subsection{Proof of Theorem~\ref{thm:thresh}}

The persistence error is $q_t-\qhat_t^{P}=q_t-q_{t-1}=\Delta_t$, so
$\mathrm{MSE}(P)=\Ecal[\Delta_t^2]=v+\mu^2$. The extrapolation error is
$q_t-\qhat_t^{E}=q_t-2q_{t-1}+q_{t-2}=\Delta_t-\Delta_{t-1}$, with mean $\mu-\mu=0$
and, by second-order stationarity,
$\Var(\Delta_t-\Delta_{t-1})=v+v-2\rho v=2v(1-\rho)$; hence
$\mathrm{MSE}(E)=2v(1-\rho)$. Comparing, $2v(1-\rho)<v+\mu^2$ iff
$\mu^2>v(1-2\rho)$. \hfill$\square$

\section{Error metrics}
\label{app:metrics}

Both error metrics compare frequency vectors accumulated over generated and real
last-snapshot tokens, and both are mean absolute deviations, so they are on a
total-variation-like scale in $[0,1]$.

\emph{Marginal error} is $\frac1D\sum_{d=1}^{D}\big|\,\hat f_d - f_d\big|$, where
$f_d$ and $\hat f_d$ are the empirical frequencies of slot $d$ being on, in the real
and generated pools respectively.

\emph{Joint error} is the same deviation computed on pairwise co-occurrence
frequencies $f_{ij}=\Prob(z_i=1 \wedge z_j=1)$. Pairs are taken as \emph{all} pairs
among the $m$ most frequently active slots (capped at a fixed budget by uniform
subsampling), which keeps the pairwise frequency estimates statistically meaningful
without selecting the pairs that would best fit the prediction.

We report ratios of these quantities --- drift period against in-period, corrected
against uncorrected, extrapolation against persistence --- rather than their absolute
values, since the metrics are ours and the slot-activity cutoff is a free parameter.
Ratios formed within a single condition are insensitive to that parameter.

\section{Per-condition measurements}
\label{app:directionality}

\begin{table}[h]
  \centering
  \caption{The seven well-powered conditions of Section~\ref{sec:universality}.
  ``Floor ratio'' is the drift-period joint-error floor divided by the in-period floor
  for the same model; ``spread over $N$'' is the relative range of the drift-period
  marginal error across $N\in\{1,\dots,50\}$, which Corollary~\ref{cor:floor} predicts
  to be approximately zero.}
  \label{tab:universality}
  \input{tables/tab1_universality}
\end{table}

\begin{table}[h]
  \centering
  \caption{Drift directionality, measured on token marginals with no model
  ($30$ chronological windows). Every condition is mean-reverting, trendless, and
  strictly harmed by extrapolation. These are the per-condition values summarised in
  Section~\ref{sec:theorem-b}.}
  \label{tab:directionality}
  \input{tables/tab2_directionality}
\end{table}

\section{Pre-registered criteria that we amended}
\label{app:protocol}

We fixed decision rules before running each experiment. Three had design defects that
became visible only once results were in, and we amended them after the fact. We did
not restate earlier verdicts under the amended rules, and we report the science by
order of magnitude rather than by whether a threshold was crossed.

\begin{itemize}
\item The sampler-step criterion compared \emph{absolute} error drops across periods.
  When the in-period drop is near zero, the ratio test passes trivially and reads noise
  as signal; it was changed to compare relative drops against a noise floor.
\item The oracle-anchoring criterion tested for ``no change'' in the in-period error
  two-sidedly, so it flagged a failure when oracle anchoring \emph{improved} the
  in-period error by $58\%$. It should have been, and is now, one-sided.
\item The directionality criterion originally rested on the
  extrapolation-to-persistence ratio alone, which we found flips with the window count
  (Appendix~\ref{app:tsweep}); it was strengthened to require a momentum signal as
  well, which is the conjunction reported in Table~\ref{tab:directionality}.
\end{itemize}

\section{Robustness of the directionality measurement to the window count}
\label{app:tsweep}

Table~\ref{tab:tsweep} recomputes the drift-process statistics of
Section~\ref{sec:theorem-b} for a range of chronological window counts $T$, over the
five conditions of Table~\ref{tab:directionality}. Section~\ref{sec:discussion}
discusses what is and is not stable across this sweep.

\begin{table}[h]
  \centering
  \caption{Drift-process statistics as a function of the number of chronological
  windows $T$. Ranges are over the five conditions of
  Table~\ref{tab:directionality}. Extrapolation beats persistence on the deployment
  period only at the coarsest binning; the momentum required by \eqref{eq:thresh},
  $\rho>1/2$, is never approached.}
  \label{tab:tsweep}
  \input{tables/tab3_tsweep}
\end{table}

%% file: tables/tab1_universality.tex
\begin{tabular}{llrrrrr}
\toprule
Domain & $N_{\max}$ & $n_{\rm val}$ & \multicolumn{2}{c}{joint-error floor} & floor & marginal err.\ \\
\cmidrule(lr){4-5}
 & & & in-period & drift & ratio & spread over $N$ \\
\midrule
Enron & 8 & 1500 & 0.0117 & 0.1676 & $14.3\times$ & 1.0\% \\
Enron & 16 & 1500 & 0.0298 & 0.2477 & $8.3\times$ & 4.5\% \\
Enron & 24 & 1500 & 0.0430 & 0.2479 & $5.8\times$ & 1.8\% \\
Enron & 32 & 1500 & 0.0682 & 0.2182 & $3.2\times$ & 6.0\% \\
MOOC & 16 & 1493 & 0.0029 & 0.0389 & $13.4\times$ & 5.5\% \\
Reddit & 16 & 1500 & 0.0395 & 0.0857 & $2.2\times$ & 4.4\% \\
Wikipedia & 16 & 1500 & 0.0014 & 0.0492 & $34.3\times$ & 3.2\% \\
\bottomrule
\end{tabular}

%% file: tables/tab2_directionality.tex
\begin{tabular}{llrrrc}
\toprule
Domain & $N_{\max}$ & increment & sign & extrap.\ / & directional? \\
 & & autocorr.\ $\rho$ & persistence & persist.\ error & \\
\midrule
Enron & 16 & $-0.19$ & 0.28 & $1.27\times$ & no \\
Enron & 32 & $-0.17$ & 0.18 & $1.28\times$ & no \\
MOOC & 16 & $-0.11$ & 0.21 & $1.35\times$ & no \\
Reddit & 16 & $-0.29$ & 0.17 & $1.51\times$ & no \\
Wikipedia & 16 & $-0.14$ & 0.12 & $1.28\times$ & no \\
\bottomrule
\end{tabular}

%% file: tables/tab3_tsweep.tex
\begin{tabular}{rrrrrr}
\toprule
windows & \multicolumn{2}{c}{increment autocorr.\ $\rho$} & max sign & \multicolumn{2}{c}{deployment extrap./persist.} \\
\cmidrule(lr){2-3}\cmidrule(lr){5-6}
$T$ & min & max & persistence & min & max \\
\midrule
10 & $-0.55$ & $-0.45$ & 0.26 & $0.84\times$ & $1.11\times$ \\
12 & $-0.51$ & $-0.41$ & 0.33 & $1.11\times$ & $1.41\times$ \\
15 & $-0.58$ & $-0.39$ & 0.33 & $1.11\times$ & $1.46\times$ \\
20 & $-0.72$ & $-0.47$ & 0.20 & $1.00\times$ & $1.24\times$ \\
25 & $-0.47$ & $-0.21$ & 0.29 & $1.06\times$ & $1.37\times$ \\
30 & $-0.29$ & $-0.11$ & 0.28 & $1.27\times$ & $1.51\times$ \\
40 & $-0.01$ & $+0.02$ & 0.21 & $1.15\times$ & $1.43\times$ \\
50 & $+0.15$ & $+0.22$ & 0.21 & $1.24\times$ & $1.53\times$ \\
\bottomrule
\end{tabular}